%% file: bare_jrnl.tex
\documentclass[journal]{IEEEtran}
\input{preamble}

\ifCLASSINFOpdf
\else
   \usepackage[dvips]{graphicx}
\fi
\usepackage{url}

\usepackage{graphicx}

\begin{document}

\title{Asymptotic Gaussian Fluctuations of Eigenvectors in Spectral Clustering}

\author{Hugo Lebeau, Florent Chatelain, Romain Couillet
%\thanks{This paragraph of the first footnote will contain the date on which you submitted your paper for review. It will also contain support information, including sponsor and financial support acknowledgment. For example, ``This work was supported in part by the U.S. Department of Commerce under Grant BS123456.''}
\thanks{H.\ Lebeau and R.\ Couillet are with Univ.\ Grenoble Alpes, CNRS, Inria, Grenoble INP, LIG, 38\,000 Grenoble, France, (e-mail: \{hugo.lebeau, romain.couillet\}@univ-grenoble-alpes.fr).}
\thanks{F. Chatelain is with Univ.\ Grenoble Alpes, CNRS, Grenoble INP, GIPSA-lab, 38\,000 Grenoble, France (e-mail: florent.chatelain@grenoble-inp.fr).}}

%\markboth{IEEE Signal Processing Letters, Vol. X, No. X, Month YEAR}{}
\maketitle

\begin{abstract}
The performance of spectral clustering relies on the fluctuations of the entries of the eigenvectors of a similarity matrix, which has been left uncharacterized until now. In this letter, it is shown that the \textit{signal $+$ noise} structure of a general spike random matrix model is transferred to the eigenvectors of the corresponding Gram kernel matrix and the fluctuations of their entries are Gaussian in the large-dimensional regime. This CLT-like result was the last missing piece to precisely predict the classification performance of spectral clustering. The proposed proof is very general and relies solely on the rotational invariance of the noise. Numerical experiments on synthetic and real data illustrate the universality of this phenomenon.
\end{abstract}

\begin{IEEEkeywords}
Spectral clustering, central limit theorem, kernel matrix, spike eigenvector, Gaussian fluctuations.
\end{IEEEkeywords}

\IEEEpeerreviewmaketitle

\section{Introduction}

\IEEEPARstart{S}{pectral} clustering is a popular unsupervised classification technique which finds applications in many domains, such as image segmentation \cite{shi_normalized_2000}, text mining \cite{brew_spectral_2002}, and as a general purpose method for data analysis \cite{ng_spectral_2002, von_luxburg_tutorial_2007, ding_min-max_2001}. It relies on the spectrum of a suitably chosen similarity matrix to perform dimensionality reduction before applying a standard clustering algorithm such as $K$-means. Consider, e.g., the following toy example where $n$ vectors $\bx_1, \ldots, \bx_n \in \bbR^p$ are separated in two clusters $\calC^+, \calC^-$ centered around $+\bmu, -\bmu$ respectively, i.e., $\bx_i = \pm \bmu + \bw_i$ where $\bw_i \sim \calN(\bzero, \bI_n)$. Then, the dominant eigenvector $\hat{\bv}$ of the Gram kernel matrix $\bK = \frac{1}{p} [ \bx_i^\top \bx_j ]_{1 \leqslant i, j \leqslant n}$ is an information-theoretically optimal estimator \cite{onatski_asymptotic_2013} of the vector $\frac{1}{\sqrt{n}} \bj$ such that $j_i = \pm1$ if $\bx_i \in \calC^\pm$. In this case, clustering is achieved with the trivial decision rule $\bx_i \to \calC^\pm$ if $\hat{v}_i \gtrless 0$.

The achievable performances of spectral clustering can be theoretically predicted thanks to the study of random matrix models corresponding to similarity matrices. For this purpose, random matrix theory offers powerful tools \cite{bai_spectral_2010, pastur_eigenvalue_2011, couillet_random_2022}. In particular, it allows to derive the limiting spectral distribution of the kernel matrix and to predict the position of isolated eigenvalues in \emph{spiked} random matrix models \cite{baik_eigenvalues_2006, benaych-georges_eigenvalues_2011, couillet_kernel_2016}. The latter are of particular importance as, in a wide range of problems, the information of interest can be modeled as a low-rank signal corrupted with noise. In our previous toy example, the data matrix $\bX = \begin{bmatrix} \bx_1 & \ldots & \bx_n \end{bmatrix}$ is a rank-one perturbation $\bmu \bj^\top$ of a noise matrix $\bW$ with i.i.d.\ $\calN(0, 1)$ entries. In order to theoretically predict the error rate of spectral clustering for a given signal-to-noise ratio, one must therefore study the behavior of the dominant eigenvectors of the similarity matrix. Tools such as the ones used in \cite{couillet_kernel_2016, couillet_two-way_2021} allow to express the quality of their alignment with the true underlying signal, i.e., $\frac{1}{\sqrt{n}} \lvert \bj^\top \hat{\bv} \rvert$. Although this tells us when an estimation of the signal is possible (depending on the signal-to-noise ratio) and its efficiency, a precise characterization of the fluctuations of the entries of spiked eigenvectors still lacks to rigorously predict the error rate of spectral clustering. Indeed, in our toy example, the expected error rate $\bbP(\hat{v}_i j_i < 0)$ cannot be expressed unless the law of $\hat{\bv}$ is known.

Yet, it is often stated that the entries of $\hat{\bv}$ have Gaussian fluctuations in the large-dimensional regime, so that $\bbP(\hat{v}_i j_i < 0)$ is a Gaussian integral. In \cite{kadavankandy_asymptotic_2019}, this result is formally stated but no proof is given. Hence, we fill this missing gap with a rigorous proof of this phenomenon for a general spiked random matrix model. Although we stick to a simple \textit{signal $+$ noise} model here, the proposed proof is not restricted to Gaussian noise (in fact, the noise only needs to be rotationally invariant) and can easily be adapted to most standard spiked models (such as, notably, the general model considered in \cite{benaych-georges_eigenvalues_2011}). Our result and its proof thus support a wide range of previous works studying the performance of spectral algorithms. The demonstration can be summarized in two simple facts
\begin{enumerate*}
\item an eigenvector of the kernel matrix can be decomposed into the sum of a deterministic signal part and a random noise part
\item the random part is uniformly distributed on a certain sphere, hence any finite subset of its entries tends to a centered Gaussian vector in the large-dimensional limit.
\end{enumerate*}

In this letter, we consider a general \textit{signal $+$ noise} random matrix model and briefly recall known results regarding its limiting spectral distribution and the behavior of its dominant eigenvalues and eigenvectors. Then, we show that the entries of the kernel eigenvectors indeed have Gaussian fluctuations in the large-dimensional regime. We present a short, self-contained and general proof which is our main contribution. Finally, we illustrate this result with numerical experiments on synthetic and real data.

\textbf{Simulations.} Python codes to reproduce simulations are available in the following GitHub repository \url{https://github.com/HugoLebeau/asymptotic_fluctuations_spectral_clustering}.

\section{Model and Main Result} \label{sec:model}

\subsection{Notations}

The symbols $a$, $\ba$ and $\bA$ respectively denote a scalar, a vector and a matrix. Their entries are $a_i$ and $A_{i, j}$. The set of positive integers below $n$ is $[n] = \set{1, \ldots, n}$. The cardinality of a set $\calE$ is $\abs{\calE}$. Given an ordered set of indices $\calI = (i_1, \ldots, i_{\abs{\calI}})$, $[\ba]_\calI$ is the vector $\begin{bmatrix} a_{i_1} & \ldots & a_{i_{\abs{\calI}}} \end{bmatrix}^\top$. The norm of a vector $\bx \in \bbR^n$ is $\norm{\bx} = \sqrt{\bx^\top \bx}$. The unit sphere in $\bbR^n$ is $\bbS^{n - 1} = \set{\bx \in \bbR^n \mid \norm{\bx} = 1}$. The $n \times n$ identity matrix is $\bI_n$. For a real number $x \in \bbR$, $[x]^+ = \max(0, x)$ and $\delta_x$ is the Dirac measure at $x$. The imaginary unit is denoted $\rmi$. If the random variable $X$ follows the law $\calL$, we write $X \sim \calL$. The convergence in distribution of a sequence of random variables $(X_n)_{n \geqslant 0}$ to $\calL$ is denoted $X_n \xrightarrow[n \to +\infty]{\calD} \calL$. Its almost sure convergence to $L$ is denoted $X_n \xrightarrow[n \to +\infty]{\text{a.s.}} L$. The multivariate normal distribution with mean $\bmu$ and covariance $\bSigma$ is denoted $\calN(\bmu, \bSigma)$. The set of eigenvalues of a square matrix $\bA$ is its spectrum, $\Sp \bA$. Given two real-valued sequences $(u_n)_{n \geqslant 0}$ and $(v_n)_{n \geqslant 0}$, we write $u_n = \O(v_n)$ if $\abs{u_n / v_n}$ is bounded as $n \to +\infty$ and $u_n \asymp v_n$ if $u_n / v_n \to 1$.

\subsection{Spiked Matrix Model}

Consider the following statistical model
\begin{equation} \label{eq:model}
\bX = \bP + \bW ~ \in \bbR^{p \times n}, \quad \bP = \bL \bV^\top
\end{equation}
with $\bL \in \bbR^{p \times K}$ and $\bV = \begin{bmatrix} \bv_1 & \dots & \bv_K \end{bmatrix} \in \bbR^{n \times K}$ such that $\bV^\top \bV = \bI_K$. It models a \emph{low-rank} signal $\bP$ corrupted by additive Gaussian noise $W_{i, j} \simiid \calN(0, 1)$. In a spectral clustering perspective, $K$ represents the number of classes and $\bP = \bM \bJ^\top$ where $\bM = \begin{bmatrix} \bmu_1 & \dots & \bmu_K \end{bmatrix}$ is a matrix gathering the $K$ cluster means and $J_{i, k} = 1$ if $\bx_i$ is in the $k$-th cluster (i.e., $\bx_i = \bmu_k + \bw_i$) and $0$ otherwise. This is congruent with model \eqref{eq:model}: define the $K \times K$ diagonal matrix $\bD$ such that $D_{k, k} = n_k$ where $n_k$ is the number of samples belonging to the $k$-th cluster, then $\bL = \bM \bD^{1 / 2}$ and $\bV = \bJ \bD^{-1 / 2}$.

Given model \eqref{eq:model}, we are interested in the reconstruction of $\bV$ from the dominant eigenvectors of the Gram kernel matrix $\bK = \frac{1}{p} \bX^\top \bX$. We study this problem in the regime where $K$ is fixed and $p, n \to +\infty$ at the same rate, i.e., $0 < c \eqdef \lim p / n < +\infty$. This models the fact that, in practice, the number of samples $n$ is comparable to the number of features $p$ and they are both large. Moreover, we make the following assumptions.
\begin{assumption} \label{ass:class_size}
All classes are of comparable size, i.e., $\liminf n_k / n > 0$ as $p, n \to +\infty$ for all $k \in [K]$.
\end{assumption}
\begin{assumption} \label{ass:non-sparsity}
$\displaystyle \lim_{p, n \to +\infty} \max_{\substack{1 \leqslant i \leqslant n \\ 1 \leqslant k \leqslant K}} \sqrt{n} V_{i, k}^2 = 0$.
\end{assumption}
\begin{assumption} \label{ass:multiplicity}
As $n \to +\infty$, the eigenvalues $\ell_1 \geqslant \ldots \geqslant \ell_K > 0$ of $\frac{1}{n} \bL^\top \bL$ are not degenerate (i.e., have multiplicity one) and the columns of $\bV$ are ordered accordingly.
\end{assumption}
Assumption \ref{ass:multiplicity} is only to simplify the presentation of the results so it is not necessary, but often verified in practice. However, Assumption \ref{ass:non-sparsity} states that $\bV$ must be \emph{delocalized}, i.e., not sparse. It is naturally verified for spectral clustering as a result of Assumption \ref{ass:class_size} since $\bV = \bJ \bD^{-1 / 2}$, but the results presented below concern the statistical model \eqref{eq:model}, which is more general and also encompasses PCA for example \cite{couillet_two-way_2021}.

\subsection{Eigenvalue Distribution and Spiked Eigenvalues}

We briefly recall known results on model \eqref{eq:model} in order to set the ground for our main result in Theorem \ref{thm:clt}.

Firstly, the empirical spectral distribution of $\bK$, that is $\frac{1}{n} \sum_{\lambda \in \Sp \bK} \delta_\lambda$, converges weakly almost surely to the Mar\v{c}enko-Pastur distribution $\mu_{\text{MP}} = [1 - c]^+ \delta_0 + \nu$ where $\nu$ has density supported on $[E_-, E_+]$ with $E_\pm = (1 \pm \sqrt{c^{-1}})^2$ \cite{marcenko_distribution_1967, bai_spectral_2010, pastur_eigenvalue_2011, couillet_random_2022}. In other words, as $p, n \to +\infty$, the histogram of eigenvalues of $\bK$ approaches $\mu_{\text{MP}}$, as depicted in Figure \ref{fig:lsd}.

\begin{figure}
\centering
\input{lsd}
\caption{Empirical Spectral Distribution (ESD) of $\bK = \frac{1}{p} \bX^\top \bX$ and Mar\v{c}enko-Pastur Distribution (MP). The green dashed lines are the positions $\xi_k$ of isolated eigenvalues predicted by Theorem \ref{thm:spikes}. \textbf{Experimental setting}: $n = 1000$, $p = 2000$, $K = 3$, $(n_1, n_2, n_3) = (333, 334, 333)$, $(\norm{\bmu_1}, \norm{\bmu_2}, \norm{\bmu_3}) = (3, 4, 5)$.}
\label{fig:lsd}
\end{figure}

Due to the low-rank perturbation $\bP$, the $K$ dominant eigenvalues of $\bK$ may isolate themselves from the bulk characterized by the Mar\v{c}enko-Pastur distribution if their corresponding signal-to-noise ratios (the eigenvalues of $\frac{1}{n} \bL^\top \bL$) are large enough --- they are then called \emph{spikes}. Their behavior is specified in the following theorem, which is a particular case of Theorem 2 in \cite{couillet_two-way_2021}.
\begin{theorem}[Spikes] \label{thm:spikes}
Let $(\lambda_k, \hat{\bv}_k)_{k \in [K]}$ denote the dominant eigenvalue-eigenvector pairs of $\bK$ such that $\lambda_1 \geqslant \ldots \geqslant \lambda_K$. Then, for all $k \in [K]$,
\begin{align*}
\lambda_k &\xrightarrow[p, n \to +\infty]{\text{a.s.}} \xi_k \eqdef \left\{ \begin{array}{ll}
\frac{(\ell_k + c) (\ell_k + 1)}{\ell_k c} & \text{if}~ \ell_k > \sqrt{c} \\
E_+ & \text{otherwise}
\end{array} \right., \\
\Abs{\bv_k^\top \hat{\bv}_k}^2 &\xrightarrow[p, n \to +\infty]{\text{a.s.}} \zeta_k \eqdef \left\{ \begin{array}{ll}
1 - \frac{\ell_k + c}{\ell_k (\ell_k + 1)} & \text{if}~ \ell_k > \sqrt{c} \\
0 & \text{otherwise}
\end{array} \right. .
\end{align*}
\end{theorem}
This result states that $\lambda_k$ leaves the bulk if, and only if, $\ell_k > \sqrt{c}$ (this is a well-known phase transition phenomenon \cite{baik_phase_2005}) and further gives its almost sure asymptotic position $\xi_k$. This is illustrated in Figure \ref{fig:lsd}. Moreover, Theorem \ref{thm:spikes} indicates the almost sure asymptotic alignment $\zeta_k$ of the corresponding eigenvector $\hat{\bv}_k$ with the underlying signal $\bv_k$. This is depicted in Figure \ref{fig:eigvecs} (top row).

\begin{figure*}
\centering
\input{eigvecs}
\caption{Dominant eigenvectors of $\bK = \frac{1}{p} \bX^\top \bX$. \textbf{Top}: Coordinates of $\hat{\bv}_k$ (blue) and the underlying signal $\sqrt{\zeta_k} \bv_k$ (orange) with $\zeta_k$ given in Theorem \ref{thm:spikes}. The dotted orange lines are the $\pm 1 \sigma$-error curves deduced from Theorem \ref{thm:clt}. \textbf{Bottom}: Histogram of the entries of $\hat{\bv}_k - \sqrt{\zeta_k} \bv_k$ (blue) and probability density function of $\calN(0, \frac{1 - \zeta_k}{n})$ (orange). \textbf{Experimental setting}: like in Figure \ref{fig:lsd}.}
\label{fig:eigvecs}
\end{figure*}

It should be noted that the previous results are not restricted to Gaussian noise: up to a control on the moments of the distribution, they can be generalized thanks to an ``interpolation trick'' \cite[Corollary 3.1]{lytova_central_2009}. In addition, a similar spectral behavior is observed with non-i.i.d.\ noise following the realistic assumption that it is \emph{concentrated} \cite{el_karoui_concentration_2009, louart_concentration_2021}.

\subsection{Fluctuations of Spiked Eigenvectors Entries}

The convergence of $\abs{\bv_k^\top \hat{\bv}_k}^2$ to $\zeta_k$ stated in Theorem \ref{thm:spikes} is an important result which justifies the use of the dominant eigenvectors of $\bK$ as estimators of the underlying signal $\bV$. Yet, it is not enough to characterize its reconstruction performance. Indeed, the fluctuations of the entries of $\hat{\bv}_k$ must be known to fully characterize \emph{how} it is aligned with $\bv_k$.

Consider, e.g., the multi-class spectral clustering problem with $\bP = \bM \bJ^\top$. Here, $[\bv_k]_i = J_{i, k} / \sqrt{n_k}$. Hence, $\bx_i$ is classified in the $k$-th class if $[\hat{\bv}_k]_i > [\hat{\bv}_{k'}]_i$ for all $k' \neq k$. The reconstruction performance thus depends on the probability of correct classification $\bbP([\hat{\bv}_k]_i > [\hat{\bv}_{k'}]_i \mid J_{i, k} = 1)$. In the theorem below, we show that the entries of $\hat{\bv}_k$ asymptotically have Gaussian fluctuations around those of $\bv_k$ with variance $(1 - \zeta_k) / n$, as illustrated in the bottom row of Figure \ref{fig:eigvecs}.
\begin{theorem} \label{thm:clt}
For all finite ordered set of indices $\calI = (i_1, \ldots, i_{\abs{\calI}}) \subset [n]$ and $k \in [K]$,
\[
\frac{\sqrt{n} \left[ \hat{\bv}_k - \sqrt{\zeta_k} \bv_k \right]_\calI}{\sqrt{1 - \zeta_k}} \xrightarrow[p, n \to +\infty]{\calD} \calN(\bzero, \bI_{\Abs{\calI}})
\]
with $\hat{\bv}_k$ such that $\bv_k^\top \hat{\bv}_k \geqslant 0$ (otherwise, consider $-\hat{\bv}_k$).
\end{theorem}
This result invokes the quantity $\zeta_k$ introduced in Theorem \ref{thm:spikes}, which quantifies the alignment of $\hat{\bv}_k$ with $\bv_k$. Theorem \ref{thm:clt} specifies that $[\hat{\bv}_k]_\calI$ behaves like $\calN(\sqrt{\zeta_k} [\bv_k]_\calI, \frac{1 - \zeta_k}{n} \bI_{\abs{\calI}})$ in the large-dimensional regime. That is, the more $\hat{\bv}_k$ is aligned with $\bv_k$ (i.e., the closer $\zeta_k$ is to $1$), the more it concentrates around $\sqrt{\zeta_k} \bv_k$, since the variance is $(1 - \zeta_k) / n$. Furthermore, the entries of $[\hat{\bv}_k]_\calI$ are \emph{asymptotically independent} for any \emph{finite} ordered set of indices $\calI$. In the multi-class spectral clustering problem considered above, since $\hat{\bv}_k$ and $\hat{\bv}_{k'}$ are asymptotically independent if $k' \neq k$ \cite[Theorem 4]{couillet_fluctuations_2013}, Theorem \ref{thm:clt} yields
\[
\bbP([\hat{\bv}_k]_i > [\hat{\bv}_{k'}]_i \mid J_{i, k} = 1) \asymp \Phi \left( \sqrt{\frac{n}{n_k} \frac{\zeta_k}{2 - (\zeta_k + \zeta_{k'})}} \right)
\]
where $\Phi : x \mapsto \frac{1}{\sqrt{2 \pi}} \int_{-\infty}^x e^{-\frac{t^2}{2}} \rmd t$ is the Gaussian cumulative distribution function.

We prove Theorem \ref{thm:clt} in Section \ref{sec:proof} below. The proof hinges on the rotational invariance of the noise (Lemma \ref{lem:rotinv}). In fact, it does not need the entries of $\bW$ to be distributed according to the Gaussian law, but only that its distribution be invariant under isometries. This makes it a very general proof, which can easily be adapted to most standard spiked models as those discussed, e.g., in \cite[\S 2.5.4]{couillet_random_2022}.

\section{Proof of Main Result} \label{sec:proof}

Consider the tangent-normal decomposition
\begin{equation} \label{eq:decomp}
\hat{\bv}_k = \sum_{\kappa = 1}^K \tau_\kappa \bv_\kappa + \sqrt{1 - \Norm{\btau}^2} \, \hat{\bv}_k^\sharp
\end{equation}
where $\hat{\bv}_k^\sharp = (\bI_n - \bV \bV^\top) \frac{\hat{\bv}_k}{\sqrt{1 - \norm{\btau}^2}}$ is a unit-norm vector orthogonal to the span of $\bV$ and $\btau = \begin{bmatrix} \tau_1 & \ldots & \tau_K \end{bmatrix}^\top$ with $\tau_\kappa = \bv_\kappa^\top \hat{\bv}_k$ measuring the cosine between $\bv_\kappa$ and $\hat{\bv}_k$. Let $\bO$ be an $n \times n$ orthogonal matrix such that $\bO \bV = \bV$ --- i.e., a rotational symmetry about the span of $\bV$ --- and $\tilde{\bK} \eqdef \bO \bK \bO^\top$. Then, since $\bK = \frac{1}{p} \bX^\top \bX$ and $\bX = \bL \bV^\top + \bW$,
\begin{multline*}
\tilde{\bK} = \frac{1}{p} \left( \left[ \bO \bV \right] \bL^\top \bL \left[ \bO \bV \right]^\top + \left[ \bO \bV \right] \bL^\top \left[ \bW \bO^\top \right] \right. \\
\left. + \left[ \bW \bO^\top \right]^\top \bL \left[ \bO \bV \right]^\top + \left[ \bW \bO^\top \right]^\top \left[ \bW \bO^\top \right] \right).
\end{multline*}

\begin{lemma} \label{lem:rotinv}
$\bW$ and $\bW \bO^\top$ are identically distributed.
\end{lemma}
\begin{proof}
The distribution of $[\bW \bO^\top]_{i, j} = \sum_{k = 1}^n W_{i, k} O_{j, k}$ is $\calN(0, 1)$ and $\Cov([\bW \bO^\top]_{i, j}, [\bW \bO^\top]_{i', j'})$ is $1$ if $(i, j) = (i', j')$ and $0$ otherwise. Hence $[\bW \bO^\top]_{i, j} \simiid \calN(0, 1)$.
\end{proof}

According to the previous lemma, $\tilde{\bK}$ follows the same model as $\bK$ since $\bO \bV = \bV$. Therefore, its $k$-th dominant eigenvector can likewise be decomposed as
\[
\tilde{\bv}_k = \sum_{\kappa = 1}^K \tilde{\tau}_\kappa \bv_\kappa + \sqrt{1 - \Norm{\tilde{\btau}}^2} \, \tilde{\bv}_k^\sharp
\]
with $\tilde{\tau}_\kappa = \bv_\kappa^\top \tilde{\bv}_k$ and $\tilde{\bv}_k^\sharp = (\bI_n - \bV \bV^\top) \frac{\tilde{\bv}_k}{\sqrt{1 - \norm{\tilde{\btau}}^2}}$ \emph{identically distributed to $\hat{\bv}_k^\sharp$}. Yet, $\tilde{\bv}_k = \bO \hat{\bv}_k$. Thus, $\hat{\bv}_k^\sharp$ and $\bO \hat{\bv}_k^\sharp$ are identically distributed for all $n \times n$ orthogonal matrix $\bO$ such that $\bO \bV = \bV$. Consequently, denoting $\eta$ the probability distribution of $\hat{\bv}_k^\sharp$ and $\bV^\perp = \set{\bw \in \bbR^n \mid \bV^\top \bw = \bzero}$, then, for all $\bx, \by \in \bbS^{n - 1} \cap \bV^\perp$, we have $\rmd \eta(\bx) = \rmd \eta(\bO \bx) = \rmd \eta(\by)$ with $\bO = \bI_n - \frac{(\bx - \by)(\bx - \by)^\top}{1 - \bx^\top \by}$ satisfying $\bO \bV = \bV$. This shows that $\hat{\bv}_k^\sharp$ is uniformly distributed on $\bbS^{n - 1} \cap \bV^\perp$.

Then, $\hat{\bv}_k^\sharp$ can be written as $\bU \bu$ where $\bu$ is uniformly distributed on $\bbS^{n - 1 - K} \subset \bbR^{n - K}$ and $\bU \in \bbR^{n \times (n - K)}$ is such that $\bU^\top \bU = \bI_{n - K}$ and $\bU^\top \bV = \bzero$ (the columns of $\bU$ form an orthonormal basis of $\bV^\perp$ in $\bbR^n$). We use the following theorem to identify the asymptotic distribution of $\sqrt{n} [\hat{\bv}_k^\sharp]_\calI$.

\begin{theorem}[\cite{schoenberg_metric_1938, steerneman_spherical_2005}] \label{thm:characteristic}
The characteristic function of a vector $\bw$ uniformly distributed on $\bbS^{n - 1}$ is given by $\varphi_\bw(\bt) \eqdef \esp{e^{\rmi \bt^\top \bw}} = \Omega_n(\norm{\bt})$ where $\Omega_n$ is such that $r \mapsto \Omega_n(r \sqrt{n})$ converges uniformly in $r \geqslant 0$ to $r \mapsto e^{-\frac{r^2}{2}}$ as $n \to +\infty$.
\end{theorem}

Let $\bt \in \bbR^n$ be such that $t_i = 0$ if $i \not\in \calI$. The characteristic function of $\sqrt{n} [\hat{\bv}_k^\sharp]_\calI$ is
\begin{align*}
\varphi_{\sqrt{n} [\hat{\bv}_k^\sharp]_\calI}(t_{i_1}, \ldots, t_{i_{\Abs{\calI}}}) &= \Esp{e^{\rmi \sqrt{n} \bt^\top \bU \bu}} \\
&= \Omega_{n - K}(\sqrt{n} \Norm{\bU^\top \bt})
\end{align*}
and $\norm{\bU^\top \bt} = \sqrt{\norm{\bt}^2 - \norm{\bV^\top \bt}^2} = \norm{\bt} + \O(\norm{\bV^\top \bt}^2)$. According to Assumption \ref{ass:non-sparsity}, $\sqrt{n} \norm{\bV^\top \bt}^2 \to 0$ as $p, n \to +\infty$, thus $\Omega_{n - K}(\sqrt{n} \norm{\bU^\top \bt}) = \Omega_{n - K}(\sqrt{n - K} \norm{\bt} + \epsilon_n)$ with $\epsilon_n \to 0$ as $p, n \to +\infty$ and
\begin{multline*}
\Abs{\Omega_{n - K}(\sqrt{n - K} \Norm{\bt} + \epsilon_n) - e^{-\frac{1}{2} \Norm{\bt}^2}} \leqslant \\
\Abs{\Omega_{n - K}(\sqrt{n - K} \Norm{\bt} + \epsilon_n) - e^{-\frac{1}{2} \left[ \Norm{\bt} + \epsilon_n / \sqrt{n - K} \right]^2}} \\
+ \Abs{e^{-\frac{1}{2} \left[ \Norm{\bt} + \epsilon_n / \sqrt{n - K} \right]^2} - e^{-\frac{1}{2} \Norm{\bt}^2}}.
\end{multline*}
As $p, n \to +\infty$, the first term vanishes from the uniform convergence given in Theorem \ref{thm:characteristic} and the second term vanishes by continuity. Therefore, $\varphi_{\sqrt{n} [\hat{\bv}_k^\sharp]_\calI}(t_{i_1}, \ldots, t_{i_{\abs{\calI}}}) \to e^{-\frac{\norm{\bt}^2}{2}}$ and, by Lévy's continuity theorem \cite{billingsley_probability_2012}, we can conclude that $\sqrt{n} [\hat{\bv}_k^\sharp]_\calI \xrightarrow{\calD} \calN(\bzero, \bI_{\abs{\calI}})$ as $p, n \to +\infty$. And, finally, given decomposition \eqref{eq:decomp}, Theorem \ref{thm:spikes} and the independence of $\hat{\bv}_k$ and $\hat{\bv}_{k'}$ if $k' \neq k$ \cite[Theorem 4]{couillet_fluctuations_2013},
\[
\hat{\bv}_k = \sqrt{\zeta_k} \bv_k + \sqrt{1 - \zeta_k} \hat{\bv}_k^\sharp + \bvarepsilon \quad \text{with} \quad \Norm{\bvarepsilon} \xrightarrow[p, n \to +\infty]{\text{a.s.}} 0.
\]
This concludes the proof of Theorem \ref{thm:clt}.

\section{Numerical Experiments} \label{sec:experiments}

To illustrate this result, we conduct a first experiment on synthetic data following model \eqref{eq:model} with $K = 3$ classes of equal size and $(\norm{\bmu_1}, \norm{\bmu_2}, \norm{\bmu_3}) = (3, 4, 5)$. The $\bx_i$'s are ordered by class. The spectral distribution of $\bK$ is plotted in Figure \ref{fig:lsd} and Figure \ref{fig:eigvecs} shows the dominant eigenvectors with the histograms of residuals $\hat{\bv}_k - \sqrt{\zeta_k} \bv_k$. We observe a very good fit of the latter to the probability density function of $\calN(0, \frac{1 - \zeta_k}{n})$ --- the $\hat{\bv}_k$'s exactly correspond to a deterministic signal $\sqrt{\zeta_k} \bv_k$ corrupted by additive centered Gaussian noise. The \textit{signal $+$ noise} structure of model \eqref{eq:model} has been transferred to the spectral estimator of $\bV$.

Then, we conduct a second experiment on the Fashion-MNIST dataset \cite{xiao_fashion-mnist_2017} consisting of $28 \times 28$ images of clothes separated in $10$ classes of size $7\,000$ each. We select two classes $k_1, k_2$ and perform binary spectral clustering using the dominant eigenvector of $\bK = \frac{1}{p} \bX^\top \bX$ where the columns of $\bX$ are the $784$ pixels of the images from classes $k_1$ and $k_2$. The dimension of $\bX$ is thus $784 \times 14\,000$. Here, we assume a similar model as our toy example in the introduction: $\bX = \bmu \bj^\top + \bW$ where $j_i = \pm1$ depending on the class of the $i$-th image. Thus, according to Theorem \ref{thm:clt}, the $i$-th entry of the dominant eigenvector $\hat{\bv}$ asymptotically follows $\calN(\sqrt{\zeta} \frac{j_i}{\sqrt{n}}, \frac{1 - \zeta}{n})$ and, given $\bj$, $\zeta$ can be estimated as $(\sum_{i = 1}^n \frac{j_i}{\sqrt{n}} \hat{v}_i)^2$. We can then compare the observed accuracy $\frac{1}{n} \sum_{i = 1}^n \bone_{\hat{v}_i j_i > 0}$ to the one expected from Theorem \ref{thm:clt}, $\bbP(\hat{v}_i j_i > 0) \asymp \Phi(\sqrt{\frac{\zeta}{1 - \zeta}})$. The results are presented in Figure \ref{fig:classification_accuracy} for each pair of classes $k_1, k_2$.

\begin{figure}
\centering
\input{classification_accuracy}
\caption{Observed (upper right, blue) and predicted (lower left, orange) classification accuracies of binary spectral clustering on the Fashion-MNIST dataset \cite{xiao_fashion-mnist_2017}.}
\label{fig:classification_accuracy}
\end{figure}

We find a very good agreement between the observed and predicted accuracies, regardless of whether the problem is easy (e.g., Trouser vs Sandal) or hard (e.g., Bag vs Ankle Boot). This observation confirms the general scope of Theorem \ref{thm:clt}: starting from real data $\bX$ which is clearly \emph{not} Gaussian, the normal distribution naturally emerges in the fluctuations of the entries of the large-dimensional eigenvector $\hat{\bv}$.

\section{Conclusion} \label{sec:conclusion}

After recalling known results on spectral clustering under a general \textit{signal $+$ noise} random matrix model, we have shown that the entries of spiked eigenvectors have Gaussian fluctuations in the large-dimensional regime. This formalizes and clearly states a result which is often implicitly assumed in many problems, without ever being actually proven. The proposed proof relies solely on the rotational invariance of the noise. It is thus very general and can easily be extended to most standard spike models. Numerical experiments have demonstrated the universality of this phenomenon: the Gaussian behavior of the entries of spike eigenvectors can even be observed on real unprocessed data. This allows to accurately predict the classification performance of spectral clustering. An interesting problem for future work is to understand how these results extend to more exotic spike models such as \cite{lebeau_random_2022}.

\newpage

\IEEEtriggeratref{14}
\bibliographystyle{IEEEtran}
\bibliography{bibliography}

\end{document}

%% file: preamble.tex
\usepackage[mathscr]{euscript}
\usepackage[hidelinks]{hyperref}
\usepackage{xurl}
\usepackage{bm}
\usepackage{xcolor}
\usepackage[inline]{enumitem}
\usepackage{amsfonts, amsmath, amssymb, amsthm}
\usepackage{mathtools}

\usepackage{tikz, pgfplots}
\pgfplotsset{compat=1.18}
\usetikzlibrary{external}
\usepgfplotslibrary{groupplots}
\theoremstyle{plain}
\newtheorem{lemma}{Lemma}
\newtheorem{theorem}{Theorem}

\theoremstyle{definition}
\newtheorem{assumption}{Assumption}

\theoremstyle{remark}

\renewcommand{\O}{\operatorname*{\mathcal{O}}}

\DeclareMathOperator{\Cov}{Cov}

\DeclareMathOperator{\Sp}{Sp}

\newcommand{\eqdef}{\overset{\text{def}}{=}}
\newcommand{\simiid}{\overset{\text{i.i.d.}}{\sim}}
\newcommand{\abs}[1]{\lvert #1 \rvert}
\newcommand{\Abs}[1]{\left\lvert #1 \right\rvert}

\newcommand{\esp}[1]{\mathbb{E}[#1]}
\newcommand{\Esp}[1]{\mathbb{E} \left[ #1 \right]}
\newcommand{\norm}[1]{\lVert #1 \rVert}
\newcommand{\Norm}[1]{\left\lVert #1 \right\rVert}

\newcommand{\set}[1]{\{ #1 \}}

\newcommand{\bzero}{{\mathbf{0}}}
\newcommand{\bone}{{\mathbf{1}}}

\newcommand{\ba}{{\bm{a}}}

\newcommand{\bj}{{\bm{j}}}

\newcommand{\bt}{{\bm{t}}}
\newcommand{\bu}{{\bm{u}}}
\newcommand{\bv}{{\bm{v}}}
\newcommand{\bw}{{\bm{w}}}
\newcommand{\bx}{{\bm{x}}}
\newcommand{\by}{{\bm{y}}}

\newcommand{\bvarepsilon}{{\bm{\varepsilon}}}

\newcommand{\bmu}{{\bm{\mu}}}

\newcommand{\btau}{{\bm{\tau}}}

\newcommand{\bA}{{\bm{A}}}

\newcommand{\bD}{{\bm{D}}}

\newcommand{\bI}{{\bm{I}}}
\newcommand{\bJ}{{\bm{J}}}
\newcommand{\bK}{{\bm{K}}}
\newcommand{\bL}{{\bm{L}}}
\newcommand{\bM}{{\bm{M}}}

\newcommand{\bO}{{\bm{O}}}
\newcommand{\bP}{{\bm{P}}}

\newcommand{\bU}{{\bm{U}}}
\newcommand{\bV}{{\bm{V}}}
\newcommand{\bW}{{\bm{W}}}
\newcommand{\bX}{{\bm{X}}}

\newcommand{\bSigma}{{\bm{\Sigma}}}

\newcommand{\calC}{{\mathcal{C}}}
\newcommand{\calD}{{\mathcal{D}}}
\newcommand{\calE}{{\mathcal{E}}}

\newcommand{\calI}{{\mathcal{I}}}

\newcommand{\calL}{{\mathcal{L}}}

\newcommand{\calN}{{\mathcal{N}}}

\newcommand{\bbP}{{\mathbb{P}}}

\newcommand{\bbR}{{\mathbb{R}}}
\newcommand{\bbS}{{\mathbb{S}}}

\newcommand{\rmd}{{\mathrm{d}}}

\newcommand{\rmi}{{\mathrm{i}}}

\definecolor{C0}{HTML}{1F77B4}
\definecolor{C1}{HTML}{FF7F0E}
\definecolor{C2}{HTML}{2CA02C}
\definecolor{C3}{HTML}{D62728}
\definecolor{C4}{HTML}{9467BD}
\definecolor{C5}{HTML}{8C564B}
\definecolor{C6}{HTML}{E377C2}
\definecolor{C7}{HTML}{7F7F7F}
\definecolor{C8}{HTML}{BCBD22}
\definecolor{C9}{HTML}{17BECF}

%% file: lsd.tex
% This file was created with tikzplotlib v0.10.1.
\begin{tikzpicture}

\definecolor{darkorange25512714}{RGB}{255,127,14}
\definecolor{forestgreen4416044}{RGB}{44,160,44}
\definecolor{steelblue31119180}{RGB}{31,119,180}

\begin{axis}[
width=\linewidth,
height=.5\linewidth,
tick align=outside,
tick pos=left,
legend style={at={(0.99, 0.5)}, anchor=east, fill opacity=0.8, draw opacity=1, text opacity=1, draw=white!80!black},
xmajorgrids,
xmin=-0.191226994364706, xmax=5.97511764529467,
xtick style={color=black},
ylabel={Density},
ymajorgrids,
ymin=0, ymax=0.942156337236126,
ytick style={color=black},
clip marker paths=true
]
\addplot [draw=steelblue31119180, fill=steelblue31119180, mark=x, only marks, line width=1pt]
table{%
x  y
0.0890613983470835 0.887086324176354
0.0898930529954567 0.887086324176354
0.0916943853843254 0.887086324176354
0.0942172987907994 0.887086324176354
0.0950217468858815 0.887086324176354
0.0956746347240955 0.887086324176354
0.0989077050012912 0.887086324176354
0.0997561752543755 0.887086324176354
0.102088266814277 0.887086324176354
0.103619331723129 0.887086324176354
0.104273563860159 0.887086324176354
0.10530747859056 0.887086324176354
0.106553167524738 0.887086324176354
0.10810771448454 0.887086324176354
0.110134982910166 0.887086324176354
0.111757893988714 0.887086324176354
0.11371680040049 0.887086324176354
0.115104001995063 0.887086324176354
0.115583073140021 0.887086324176354
0.116827275029971 0.887086324176354
0.118704862041806 0.887086324176354
0.119775217301892 0.887086324176354
0.120520101379094 0.887086324176354
0.121640840950803 0.887086324176354
0.122616392226723 0.887086324176354
0.125080119901366 0.887086324176354
0.125283016853739 0.887086324176354
0.126081481101225 0.887086324176354
0.12784983020145 0.887086324176354
0.128015836049805 0.887086324176354
0.128639923433019 0.887086324176354
0.129573722556117 0.887086324176354
0.131379699350136 0.887086324176354
0.132373270857918 0.887086324176354
0.133782409275398 0.887086324176354
0.134934184445989 0.887086324176354
0.136807707939514 0.887086324176354
0.138048782229645 0.887086324176354
0.138438915585474 0.887086324176354
0.139521819589594 0.887086324176354
0.140406297339402 0.887086324176354
0.142629200932017 0.887086324176354
0.142940070205049 0.887086324176354
0.144198101541149 0.887086324176354
0.145492334101745 0.887086324176354
0.146658437497321 0.887086324176354
0.148763715541305 0.887086324176354
0.150823987673428 0.887086324176354
0.151773667010866 0.887086324176354
0.152266497353473 0.887086324176354
0.153012748351159 0.887086324176354
0.154731062577976 0.887086324176354
0.156229469134469 0.887086324176354
0.157187752821597 0.887086324176354
0.157931557848852 0.887086324176354
0.15905360331106 0.887086324176354
0.160371584437254 0.887086324176354
0.161349715278022 0.887086324176354
0.163183592725246 0.887086324176354
0.164067856914528 0.887086324176354
0.164400955206055 0.887086324176354
0.165036778501617 0.887086324176354
0.166725672742242 0.887086324176354
0.168625366006388 0.887086324176354
0.168897535439811 0.887086324176354
0.170295159338791 0.887086324176354
0.171863457598144 0.887086324176354
0.172918460010844 0.887086324176354
0.173299810886793 0.887086324176354
0.175543949100702 0.887086324176354
0.175715002558625 0.887086324176354
0.177709558519233 0.887086324176354
0.178952411648715 0.887086324176354
0.179401519354257 0.887086324176354
0.181215408665497 0.887086324176354
0.182201149145265 0.887086324176354
0.183958611288739 0.887086324176354
0.184217440953966 0.887086324176354
0.185002557888692 0.887086324176354
0.185977518410114 0.887086324176354
0.187088754442275 0.887086324176354
0.187558425745599 0.887086324176354
0.188504701825624 0.887086324176354
0.191261100633338 0.887086324176354
0.192988058511894 0.887086324176354
0.193034313178928 0.887086324176354
0.194000597529526 0.887086324176354
0.195823705840135 0.887086324176354
0.195865625692533 0.887086324176354
0.197095411782042 0.887086324176354
0.198636895839989 0.887086324176354
0.199219393892038 0.887086324176354
0.200281990378637 0.887086324176354
0.201403733184937 0.887086324176354
0.202552526665935 0.887086324176354
0.204602239651147 0.887086324176354
0.205005220143744 0.887086324176354
0.207237798528751 0.887086324176354
0.207555506947988 0.887086324176354
0.208195167078625 0.887086324176354
0.210010277188884 0.887086324176354
0.210454464266033 0.887086324176354
0.210944659067791 0.887086324176354
0.212956749239317 0.887086324176354
0.213875139894289 0.887086324176354
0.214707253031919 0.887086324176354
0.216526955104601 0.887086324176354
0.218320680456273 0.887086324176354
0.22017324929626 0.887086324176354
0.221603336079646 0.887086324176354
0.22177185258302 0.887086324176354
0.223439875614994 0.887086324176354
0.224781659802309 0.887086324176354
0.225703658309764 0.887086324176354
0.227409034091706 0.887086324176354
0.227954569321059 0.887086324176354
0.229470060760403 0.887086324176354
0.230083413178139 0.887086324176354
0.231068018640531 0.887086324176354
0.232005554532139 0.887086324176354
0.232899163749525 0.887086324176354
0.234111771502819 0.887086324176354
0.236173607393478 0.887086324176354
0.236863436391944 0.887086324176354
0.238928230605646 0.887086324176354
0.23974390721726 0.887086324176354
0.24039486622018 0.887086324176354
0.240896519274092 0.887086324176354
0.2426377537896 0.887086324176354
0.24442911411194 0.887086324176354
0.244588148647968 0.887086324176354
0.246293389254542 0.887086324176354
0.24711198523898 0.887086324176354
0.248049760960955 0.887086324176354
0.249357780929438 0.887086324176354
0.250953078863347 0.887086324176354
0.251372855575977 0.887086324176354
0.252720304045603 0.887086324176354
0.254526850880125 0.887086324176354
0.255794755504417 0.887086324176354
0.25654009266443 0.887086324176354
0.258369704490396 0.887086324176354
0.259394228784632 0.887086324176354
0.259771840819258 0.887086324176354
0.260755107576831 0.887086324176354
0.262629064551611 0.887086324176354
0.263549144741175 0.887086324176354
0.263721851775724 0.887086324176354
0.265315196463789 0.887086324176354
0.266530084323054 0.887086324176354
0.26880165075979 0.887086324176354
0.26979240996508 0.887086324176354
0.270354783008282 0.887086324176354
0.27260544620064 0.887086324176354
0.274279394493899 0.887086324176354
0.275557228260128 0.887086324176354
0.276279864346298 0.887086324176354
0.277533801879883 0.887086324176354
0.278249934597357 0.887086324176354
0.278615783783925 0.887086324176354
0.280300015528451 0.887086324176354
0.281267683432448 0.887086324176354
0.282715817855146 0.887086324176354
0.284652972409329 0.887086324176354
0.286488457408047 0.887086324176354
0.286961710210207 0.887086324176354
0.287203922693542 0.887086324176354
0.288111623138092 0.887086324176354
0.289077014219618 0.887086324176354
0.290692065565823 0.887086324176354
0.292481566091775 0.887086324176354
0.294183281811305 0.887086324176354
0.295621199540921 0.887086324176354
0.297660456021077 0.887086324176354
0.298659825197877 0.887086324176354
0.299575956739517 0.887086324176354
0.301223153073488 0.887086324176354
0.301927150297941 0.887086324176354
0.303966112568897 0.887086324176354
0.305104263060669 0.887086324176354
0.306687349483033 0.887086324176354
0.309061812919828 0.887086324176354
0.309924330921807 0.887086324176354
0.310922548160166 0.887086324176354
0.311444732210071 0.887086324176354
0.312199739650986 0.887086324176354
0.313632461166516 0.887086324176354
0.31531851018078 0.887086324176354
0.316895511833936 0.887086324176354
0.317805659426615 0.887086324176354
0.31853928537179 0.887086324176354
0.319929049893106 0.887086324176354
0.321222415333617 0.887086324176354
0.321958475685228 0.887086324176354
0.322678354588907 0.887086324176354
0.324541385532433 0.887086324176354
0.327155463099631 0.887086324176354
0.327470283929909 0.887086324176354
0.329252829931235 0.887086324176354
0.331103342116034 0.887086324176354
0.332462323153545 0.887086324176354
0.33317524276478 0.887086324176354
0.335184898393159 0.887086324176354
0.33556226838383 0.887086324176354
0.338830598614189 0.887086324176354
0.339332129097062 0.887086324176354
0.340956357034267 0.887086324176354
0.341739197267579 0.887086324176354
0.343116964539721 0.887086324176354
0.34398097562615 0.887086324176354
0.344817423224491 0.887086324176354
0.346935766447 0.887086324176354
0.347430581286684 0.887086324176354
0.348578033538726 0.887086324176354
0.349789654889687 0.887086324176354
0.352212814786842 0.887086324176354
0.354054090620546 0.887086324176354
0.355750248803441 0.887086324176354
0.356428258374835 0.887086324176354
0.35843639121606 0.887086324176354
0.358836546550617 0.887086324176354
0.36005216011195 0.887086324176354
0.361813551098831 0.887086324176354
0.362280691646931 0.887086324176354
0.364117773702018 0.887086324176354
0.365120003827971 0.887086324176354
0.366698800878396 0.887086324176354
0.366809151314091 0.887086324176354
0.367882887865065 0.887086324176354
0.37074440403693 0.887086324176354
0.373176236520032 0.887086324176354
0.373854650890056 0.887086324176354
0.375372710217916 0.887086324176354
0.377322451049681 0.887086324176354
0.377747670885688 0.887086324176354
0.378273017691159 0.887086324176354
0.380393321049234 0.887086324176354
0.381078915297554 0.887086324176354
0.382415601854174 0.887086324176354
0.384648700637018 0.887086324176354
0.386918786195152 0.887086324176354
0.387843354648048 0.887086324176354
0.389017663498004 0.887086324176354
0.389993693547892 0.887086324176354
0.391406236665895 0.887086324176354
0.39264205902782 0.887086324176354
0.39576435535731 0.887086324176354
0.396826432084641 0.887086324176354
0.397394697355459 0.887086324176354
0.399437989875203 0.887086324176354
0.400717287396729 0.887086324176354
0.400966172069749 0.887086324176354
0.402162426355739 0.887086324176354
0.404304923758292 0.887086324176354
0.406873801906846 0.887086324176354
0.407646970583403 0.887086324176354
0.408381625198791 0.887086324176354
0.409236327238181 0.887086324176354
0.411717210830707 0.887086324176354
0.411985283876135 0.887086324176354
0.412897978625288 0.887086324176354
0.415263076170613 0.887086324176354
0.417424761122362 0.887086324176354
0.417884692857616 0.887086324176354
0.419587015837886 0.887086324176354
0.421602816897803 0.887086324176354
0.42339773530889 0.887086324176354
0.425506333309007 0.887086324176354
0.425660977525399 0.887086324176354
0.426984825761447 0.887086324176354
0.429295728407538 0.887086324176354
0.431081526502239 0.887086324176354
0.431890277310896 0.887086324176354
0.433526526693298 0.887086324176354
0.436534273286551 0.887086324176354
0.437750848517679 0.887086324176354
0.438911808851684 0.887086324176354
0.440450102507287 0.887086324176354
0.441535195693459 0.887086324176354
0.443750718800903 0.887086324176354
0.443995072408342 0.887086324176354
0.444590276925748 0.887086324176354
0.447146124295293 0.887086324176354
0.448981073344886 0.887086324176354
0.449608958772922 0.887086324176354
0.4502737925505 0.887086324176354
0.452181465402811 0.887086324176354
0.453180500317861 0.887086324176354
0.454618904167076 0.887086324176354
0.456623569140221 0.887086324176354
0.457622421140213 0.887086324176354
0.459472047499403 0.887086324176354
0.461413152295086 0.887086324176354
0.463330669513364 0.887086324176354
0.464291446908 0.887086324176354
0.46467446568866 0.887086324176354
0.465467212501003 0.887086324176354
0.468983370344841 0.887086324176354
0.469887077112479 0.887086324176354
0.470659734568913 0.887086324176354
0.47330969608622 0.887086324176354
0.474371928835674 0.887086324176354
0.475999898980336 0.887086324176354
0.477098938617855 0.887086324176354
0.477909020193126 0.887086324176354
0.479579137293016 0.887086324176354
0.480669209499007 0.887086324176354
0.483602315110467 0.887086324176354
0.484961358013949 0.887086324176354
0.486428885591263 0.887086324176354
0.486876968713952 0.887086324176354
0.48923337476682 0.887086324176354
0.490942810561532 0.887086324176354
0.492061194200532 0.887086324176354
0.492707833633997 0.887086324176354
0.493884942602625 0.887086324176354
0.496474879131495 0.887086324176354
0.497722947275347 0.887086324176354
0.498927936715955 0.887086324176354
0.501040019363463 0.887086324176354
0.50317059067472 0.887086324176354
0.503265285146193 0.887086324176354
0.503902636954403 0.887086324176354
0.508554496691187 0.887086324176354
0.509996733612733 0.887086324176354
0.511515070195654 0.887086324176354
0.512953015168257 0.887086324176354
0.51475922138119 0.887086324176354
0.516852555362977 0.887086324176354
0.517352211697517 0.887086324176354
0.520287117788075 0.887086324176354
0.521058927668698 0.887086324176354
0.522578941234299 0.887086324176354
0.524444132570999 0.887086324176354
0.525970425635028 0.887086324176354
0.527213511517902 0.887086324176354
0.529566476675945 0.887086324176354
0.530711205259913 0.887086324176354
0.531271391388571 0.887086324176354
0.534719973531905 0.887086324176354
0.536573412222471 0.887086324176354
0.537155168535732 0.887086324176354
0.538450767612109 0.887086324176354
0.540492622948108 0.887086324176354
0.542430420256949 0.887086324176354
0.543140592264177 0.887086324176354
0.544985000889628 0.887086324176354
0.545770462872293 0.887086324176354
0.547681665618401 0.887086324176354
0.549030406304594 0.887086324176354
0.55078895949171 0.887086324176354
0.55190243285219 0.887086324176354
0.553270851237692 0.887086324176354
0.555134857125607 0.887086324176354
0.557024888662712 0.887086324176354
0.559071601995954 0.887086324176354
0.560990765765633 0.887086324176354
0.563728470919909 0.887086324176354
0.565843484683269 0.887086324176354
0.566606320159554 0.887086324176354
0.568754315889144 0.887086324176354
0.571708663707998 0.887086324176354
0.571878892319345 0.887086324176354
0.573418458540331 0.887086324176354
0.574792368056559 0.887086324176354
0.5759632467153 0.887086324176354
0.577185610447207 0.887086324176354
0.580079109727077 0.887086324176354
0.580406212312104 0.887086324176354
0.58382670071959 0.887086324176354
0.58554830660557 0.887086324176354
0.587092754108895 0.887086324176354
0.587703634124068 0.887086324176354
0.59049336967299 0.887086324176354
0.591698123572338 0.887086324176354
0.594119108628601 0.887086324176354
0.59602915756822 0.887086324176354
0.596932595124828 0.887086324176354
0.5989675163996 0.887086324176354
0.60035620490806 0.887086324176354
0.603067657047149 0.887086324176354
0.603889389940788 0.887086324176354
0.606305189281505 0.887086324176354
0.607667763828821 0.887086324176354
0.60912674895515 0.887086324176354
0.611640640967219 0.887086324176354
0.612542433418379 0.887086324176354
0.614684311168379 0.887086324176354
0.617576349487767 0.887086324176354
0.618803058015308 0.887086324176354
0.6205136843911 0.887086324176354
0.623011369663235 0.887086324176354
0.625507108371114 0.887086324176354
0.626000189107446 0.887086324176354
0.627049785309051 0.887086324176354
0.628981527247409 0.887086324176354
0.629116146075956 0.887086324176354
0.63302272154432 0.887086324176354
0.634543587483032 0.887086324176354
0.635364037167433 0.887086324176354
0.636568608926082 0.887086324176354
0.63862350339206 0.887086324176354
0.64088210673044 0.887086324176354
0.643109710462821 0.887086324176354
0.644536456711479 0.887086324176354
0.645098422882583 0.887086324176354
0.647670767892241 0.887086324176354
0.650843937630172 0.887086324176354
0.651549443740687 0.887086324176354
0.653661388581488 0.887086324176354
0.656000635737741 0.887086324176354
0.657416195117748 0.887086324176354
0.657842256505915 0.887086324176354
0.660159893666406 0.887086324176354
0.661739853986148 0.887086324176354
0.664843802543464 0.887086324176354
0.665646017732336 0.887086324176354
0.667711883552823 0.887086324176354
0.667943692624573 0.887086324176354
0.670736271541418 0.887086324176354
0.673552035438483 0.887086324176354
0.676341877259266 0.887086324176354
0.677630093203736 0.887086324176354
0.680241562495884 0.887086324176354
0.681252571710108 0.887086324176354
0.685025291057271 0.887086324176354
0.685874982159558 0.887086324176354
0.687336539764641 0.887086324176354
0.687795661833683 0.887086324176354
0.689912059546407 0.887086324176354
0.691863286656524 0.887086324176354
0.692895029532871 0.887086324176354
0.694649428402356 0.887086324176354
0.696730997882007 0.887086324176354
0.697654973695062 0.887086324176354
0.700259366833804 0.887086324176354
0.701526512008173 0.887086324176354
0.703880871804487 0.887086324176354
0.706373061018531 0.887086324176354
0.708664349423351 0.887086324176354
0.710294332307028 0.887086324176354
0.712620727842963 0.887086324176354
0.716580731744589 0.887086324176354
0.717124227909765 0.887086324176354
0.720193032930057 0.887086324176354
0.721751820908333 0.887086324176354
0.722612300663304 0.887086324176354
0.724504815025671 0.887086324176354
0.727169609793708 0.887086324176354
0.729142793721923 0.887086324176354
0.730679799132207 0.887086324176354
0.733157555637973 0.887086324176354
0.733963674488251 0.887086324176354
0.737874275358444 0.887086324176354
0.738674988664905 0.887086324176354
0.74028063614385 0.887086324176354
0.742152755068759 0.887086324176354
0.744538325164853 0.887086324176354
0.745683579280695 0.887086324176354
0.748152461469829 0.887086324176354
0.750319052269725 0.887086324176354
0.751524935095923 0.887086324176354
0.754048497385855 0.887086324176354
0.755076493645045 0.887086324176354
0.758437937284016 0.887086324176354
0.760967960988773 0.887086324176354
0.761587589932511 0.887086324176354
0.762357362273684 0.887086324176354
0.764931590343372 0.887086324176354
0.766084337570353 0.887086324176354
0.768401692928528 0.887086324176354
0.771011024384222 0.887086324176354
0.772838979372357 0.887086324176354
0.774270910241327 0.887086324176354
0.777588180682072 0.887086324176354
0.780668953047735 0.887086324176354
0.781056580950808 0.887086324176354
0.783031947721989 0.887086324176354
0.784949169643752 0.887086324176354
0.787936930969289 0.887086324176354
0.790915986845714 0.887086324176354
0.792236446609643 0.887086324176354
0.793255916283615 0.887086324176354
0.795045959357143 0.887086324176354
0.79753525776132 0.887086324176354
0.799125387851941 0.887086324176354
0.800327200035375 0.887086324176354
0.80271551579025 0.887086324176354
0.806295810312734 0.887086324176354
0.809583680152422 0.887086324176354
0.810209709680403 0.887086324176354
0.814233694652567 0.887086324176354
0.815612043467157 0.887086324176354
0.818526786435378 0.887086324176354
0.819567970979114 0.887086324176354
0.821397361109484 0.887086324176354
0.823342756970379 0.887086324176354
0.824667000798491 0.887086324176354
0.828389606522437 0.887086324176354
0.830681279134439 0.887086324176354
0.830881385748125 0.887086324176354
0.833425052427962 0.887086324176354
0.835619465674665 0.887086324176354
0.836717616506124 0.887086324176354
0.838047056444744 0.887086324176354
0.841145280734081 0.887086324176354
0.844969453192362 0.887086324176354
0.846258865149353 0.887086324176354
0.846927606939462 0.887086324176354
0.848197672168333 0.887086324176354
0.851334383799796 0.887086324176354
0.856738174920844 0.887086324176354
0.859930767183062 0.887086324176354
0.86054876912116 0.887086324176354
0.862508341409239 0.887086324176354
0.863952830807963 0.887086324176354
0.869610807051112 0.887086324176354
0.872800077685364 0.887086324176354
0.873067660988446 0.887086324176354
0.874936592080107 0.887086324176354
0.877088344317096 0.887086324176354
0.878145469741113 0.887086324176354
0.881477466713221 0.887086324176354
0.881674053571881 0.887086324176354
0.88385804752496 0.887086324176354
0.885162429185827 0.887086324176354
0.888192970590368 0.887086324176354
0.892142637558854 0.887086324176354
0.893084147580541 0.887086324176354
0.896171684630575 0.887086324176354
0.898496245092867 0.887086324176354
0.900077174903125 0.887086324176354
0.902568833203326 0.887086324176354
0.90518332654805 0.887086324176354
0.905928415384248 0.887086324176354
0.909376657962474 0.887086324176354
0.91137182434872 0.887086324176354
0.914150268270821 0.887086324176354
0.917423420608383 0.887086324176354
0.918903734357781 0.887086324176354
0.919242930522231 0.887086324176354
0.922925858987123 0.887086324176354
0.92380195989531 0.887086324176354
0.927915955255318 0.887086324176354
0.928632888724264 0.887086324176354
0.929993313909064 0.887086324176354
0.935726201122526 0.887086324176354
0.937743068383372 0.887086324176354
0.938480141092068 0.887086324176354
0.94012782327488 0.887086324176354
0.943046661315279 0.887086324176354
0.945862092522528 0.887086324176354
0.947561338610515 0.887086324176354
0.948782911024689 0.887086324176354
0.949806747116507 0.887086324176354
0.956351225988461 0.887086324176354
0.957237370162993 0.887086324176354
0.957954913445986 0.887086324176354
0.958409362726552 0.887086324176354
0.963647492106847 0.887086324176354
0.965242945858517 0.887086324176354
0.966434902666984 0.887086324176354
0.969759509436503 0.887086324176354
0.971211629871414 0.887086324176354
0.973842576776125 0.887086324176354
0.975335368565515 0.887086324176354
0.978197175211813 0.887086324176354
0.980776938362301 0.887086324176354
0.982461501140618 0.887086324176354
0.985279811144417 0.887086324176354
0.99021350584329 0.887086324176354
0.992541954302198 0.887086324176354
0.996007657286662 0.887086324176354
0.997039414079878 0.887086324176354
1.00023340857689 0.887086324176354
1.00105025957969 0.887086324176354
1.00110746197146 0.887086324176354
1.00771359200947 0.887086324176354
1.01025165638039 0.887086324176354
1.0110967272354 0.887086324176354
1.01423440362684 0.887086324176354
1.01688868777266 0.887086324176354
1.01927064260082 0.887086324176354
1.02038528560841 0.887086324176354
1.02250683239972 0.887086324176354
1.02673932506914 0.887086324176354
1.02828676166418 0.887086324176354
1.03157463002262 0.887086324176354
1.03444187971081 0.887086324176354
1.03679195254146 0.887086324176354
1.03841879280172 0.887086324176354
1.04046754794538 0.887086324176354
1.04137110656404 0.887086324176354
1.04355678161104 0.887086324176354
1.0449393175279 0.887086324176354
1.04868179288187 0.887086324176354
1.05383378409409 0.887086324176354
1.05695558101485 0.887086324176354
1.05940309439035 0.887086324176354
1.06126192914874 0.887086324176354
1.06339032233718 0.887086324176354
1.06583064632065 0.887086324176354
1.06965161337198 0.887086324176354
1.0722747885013 0.887086324176354
1.07435859288875 0.887086324176354
1.07533476476299 0.887086324176354
1.07736557047121 0.887086324176354
1.07940955422204 0.887086324176354
1.0824494825056 0.887086324176354
1.08604726467817 0.887086324176354
1.0866865828991 0.887086324176354
1.0911892762683 0.887086324176354
1.09158370748285 0.887086324176354
1.09346878183836 0.887086324176354
1.0958002871098 0.887086324176354
1.10113774917004 0.887086324176354
1.10234712405158 0.887086324176354
1.10381099056354 0.887086324176354
1.10718285718759 0.887086324176354
1.1089863284631 0.887086324176354
1.11322054379063 0.887086324176354
1.11690077201351 0.887086324176354
1.11860755681049 0.887086324176354
1.12013958463529 0.887086324176354
1.12142556824689 0.887086324176354
1.12517488172261 0.887086324176354
1.12876717904673 0.887086324176354
1.1293257702714 0.887086324176354
1.13572304969684 0.887086324176354
1.13774061147917 0.887086324176354
1.13916278185135 0.887086324176354
1.14252717179367 0.887086324176354
1.14461653782508 0.887086324176354
1.14672505769518 0.887086324176354
1.14991312498096 0.887086324176354
1.15238181702555 0.887086324176354
1.15781459068913 0.887086324176354
1.15987784376611 0.887086324176354
1.16182823697636 0.887086324176354
1.16187924971328 0.887086324176354
1.16381819466878 0.887086324176354
1.16882186386593 0.887086324176354
1.16963168586175 0.887086324176354
1.17123591651132 0.887086324176354
1.17493375408986 0.887086324176354
1.17624342072947 0.887086324176354
1.18027791899578 0.887086324176354
1.18329967690131 0.887086324176354
1.18585535393454 0.887086324176354
1.18952748784846 0.887086324176354
1.19318463743359 0.887086324176354
1.19661053071182 0.887086324176354
1.1991495057744 0.887086324176354
1.20031332296415 0.887086324176354
1.2036300757763 0.887086324176354
1.20517595995532 0.887086324176354
1.20839082990401 0.887086324176354
1.21319886464171 0.887086324176354
1.21443780842258 0.887086324176354
1.21824402990593 0.887086324176354
1.22155034993231 0.887086324176354
1.22241571484105 0.887086324176354
1.22366233288682 0.887086324176354
1.22943210520486 0.887086324176354
1.23245593191177 0.887086324176354
1.23525962230653 0.887086324176354
1.23670027322514 0.887086324176354
1.23703995608303 0.887086324176354
1.24046438903968 0.887086324176354
1.24231340791434 0.887086324176354
1.24429129306017 0.887086324176354
1.2491244166983 0.887086324176354
1.24969812104324 0.887086324176354
1.25334462515846 0.887086324176354
1.25561962408826 0.887086324176354
1.2580579161668 0.887086324176354
1.26363361662776 0.887086324176354
1.2665128539823 0.887086324176354
1.26746508954033 0.887086324176354
1.27174186289274 0.887086324176354
1.27584359326663 0.887086324176354
1.2775333441386 0.887086324176354
1.28144229122127 0.887086324176354
1.28459206062911 0.887086324176354
1.28472575413842 0.887086324176354
1.28886601897986 0.887086324176354
1.28953140710287 0.887086324176354
1.29467855574954 0.887086324176354
1.29724209869207 0.887086324176354
1.29878797467316 0.887086324176354
1.30051682095261 0.887086324176354
1.30607573224252 0.887086324176354
1.3122318428113 0.887086324176354
1.31471901934496 0.887086324176354
1.31621223364653 0.887086324176354
1.3198330694898 0.887086324176354
1.32089039681901 0.887086324176354
1.32547656577134 0.887086324176354
1.32703617328107 0.887086324176354
1.32853316756664 0.887086324176354
1.33243694813592 0.887086324176354
1.33699950923502 0.887086324176354
1.3380308127265 0.887086324176354
1.34189124019605 0.887086324176354
1.34387530574243 0.887086324176354
1.34573763900937 0.887086324176354
1.3493761837813 0.887086324176354
1.35469380179687 0.887086324176354
1.35539149182068 0.887086324176354
1.35841262057796 0.887086324176354
1.36229474866962 0.887086324176354
1.36634138121837 0.887086324176354
1.36932226721263 0.887086324176354
1.3737739933888 0.887086324176354
1.37631178978129 0.887086324176354
1.37824905697406 0.887086324176354
1.38259648344715 0.887086324176354
1.38412583539565 0.887086324176354
1.38584075764745 0.887086324176354
1.3862899231233 0.887086324176354
1.39219227440195 0.887086324176354
1.39626828813777 0.887086324176354
1.40063492593443 0.887086324176354
1.40524898726334 0.887086324176354
1.4089850019013 0.887086324176354
1.41060546420684 0.887086324176354
1.4170646316071 0.887086324176354
1.41929618238537 0.887086324176354
1.42138472511886 0.887086324176354
1.42232770161523 0.887086324176354
1.42486321065252 0.887086324176354
1.42636467903535 0.887086324176354
1.43556898486533 0.887086324176354
1.43812951959384 0.887086324176354
1.43879221166339 0.887086324176354
1.4411663232549 0.887086324176354
1.44447347755689 0.887086324176354
1.44722706468705 0.887086324176354
1.44978798052704 0.887086324176354
1.45241444311656 0.887086324176354
1.45473969363419 0.887086324176354
1.45850291556479 0.887086324176354
1.4607793166376 0.887086324176354
1.46240233942471 0.887086324176354
1.46434713454923 0.887086324176354
1.46730657527025 0.887086324176354
1.47256199586529 0.887086324176354
1.48062375538769 0.887086324176354
1.48228815194807 0.887086324176354
1.48655607737675 0.887086324176354
1.48859108519741 0.887086324176354
1.49430568373006 0.887086324176354
1.49493394998421 0.887086324176354
1.50241922476915 0.887086324176354
1.50414077948103 0.887086324176354
1.50804774399532 0.887086324176354
1.51163651160708 0.887086324176354
1.51351532698435 0.887086324176354
1.5182076173326 0.887086324176354
1.5238760411113 0.887086324176354
1.52633561386685 0.887086324176354
1.53006926923197 0.887086324176354
1.53347266011115 0.887086324176354
1.53688336976049 0.887086324176354
1.53975884141473 0.887086324176354
1.5412929810777 0.887086324176354
1.5434321718043 0.887086324176354
1.54599522282642 0.887086324176354
1.54705440115092 0.887086324176354
1.55463094991551 0.887086324176354
1.55540795140431 0.887086324176354
1.55741086383097 0.887086324176354
1.56111761680057 0.887086324176354
1.56568794594041 0.887086324176354
1.5712219464785 0.887086324176354
1.57442813767183 0.887086324176354
1.57528883362017 0.887086324176354
1.58091350431353 0.887086324176354
1.58521270525521 0.887086324176354
1.58674707048945 0.887086324176354
1.58956044436834 0.887086324176354
1.59344452633225 0.887086324176354
1.59560442554888 0.887086324176354
1.60262250282543 0.887086324176354
1.60695472814271 0.887086324176354
1.60765779909495 0.887086324176354
1.61234769601147 0.887086324176354
1.61537316033078 0.887086324176354
1.62009658458929 0.887086324176354
1.62566252471848 0.887086324176354
1.62726081193129 0.887086324176354
1.6309075482506 0.887086324176354
1.63733122244333 0.887086324176354
1.64018386806144 0.887086324176354
1.64472708788384 0.887086324176354
1.64943571461533 0.887086324176354
1.65340408255058 0.887086324176354
1.65608351982922 0.887086324176354
1.65890107507196 0.887086324176354
1.66213280129184 0.887086324176354
1.66634355780016 0.887086324176354
1.67192490013899 0.887086324176354
1.67481886947166 0.887086324176354
1.67923821753979 0.887086324176354
1.68463111892902 0.887086324176354
1.68774665702853 0.887086324176354
1.68857247667072 0.887086324176354
1.69266190145442 0.887086324176354
1.69838696734136 0.887086324176354
1.70383002622813 0.887086324176354
1.70595373508694 0.887086324176354
1.70850415144771 0.887086324176354
1.71387966680633 0.887086324176354
1.71501845782873 0.887086324176354
1.72099191821974 0.887086324176354
1.7266816610947 0.887086324176354
1.7282189486198 0.887086324176354
1.73293871505412 0.887086324176354
1.73429175237393 0.887086324176354
1.73965277752698 0.887086324176354
1.74315584547312 0.887086324176354
1.74388441049099 0.887086324176354
1.74910405461827 0.887086324176354
1.75259205859538 0.887086324176354
1.75818427429717 0.887086324176354
1.76289690561955 0.887086324176354
1.76483297225446 0.887086324176354
1.76808401571612 0.887086324176354
1.77296189457881 0.887086324176354
1.77836870412326 0.887086324176354
1.78528925619279 0.887086324176354
1.78607722183742 0.887086324176354
1.79098550271628 0.887086324176354
1.7948773083027 0.887086324176354
1.79754838251042 0.887086324176354
1.80347200758 0.887086324176354
1.80826882544711 0.887086324176354
1.81240561911191 0.887086324176354
1.81461769308831 0.887086324176354
1.82078859993669 0.887086324176354
1.82528794403403 0.887086324176354
1.82801351355463 0.887086324176354
1.83134182528728 0.887086324176354
1.83951154530733 0.887086324176354
1.84362282588828 0.887086324176354
1.84522239970136 0.887086324176354
1.849627689665 0.887086324176354
1.85531283545945 0.887086324176354
1.85947762094744 0.887086324176354
1.8658043161706 0.887086324176354
1.8712112963946 0.887086324176354
1.8749814268147 0.887086324176354
1.87727715454213 0.887086324176354
1.88120516973577 0.887086324176354
1.88528321712026 0.887086324176354
1.88966132051264 0.887086324176354
1.89386217710897 0.887086324176354
1.89609053228469 0.887086324176354
1.89940943219589 0.887086324176354
1.90303296399418 0.887086324176354
1.90621735405455 0.887086324176354
1.9123935088266 0.887086324176354
1.91953828684051 0.887086324176354
1.92683977402563 0.887086324176354
1.93108119461172 0.887086324176354
1.9350300822911 0.887086324176354
1.93888271548509 0.887086324176354
1.94564202911415 0.887086324176354
1.94904591221723 0.887086324176354
1.95381005159975 0.887086324176354
1.95959155803815 0.887086324176354
1.96058104562257 0.887086324176354
1.96789372932331 0.887086324176354
1.97000525865021 0.887086324176354
1.97358059725692 0.887086324176354
1.97688782817389 0.887086324176354
1.98357712472451 0.887086324176354
1.9878672516475 0.887086324176354
1.99042638131354 0.887086324176354
2.00051691739207 0.887086324176354
2.00363114201069 0.887086324176354
2.00912138146336 0.887086324176354
2.01386564587719 0.887086324176354
2.01708254896712 0.887086324176354
2.02707174895271 0.887086324176354
2.03078438694412 0.887086324176354
2.03282086911162 0.887086324176354
2.04035720457222 0.887086324176354
2.04607549968937 0.887086324176354
2.047637656233 0.887086324176354
2.05372624705002 0.887086324176354
2.05815466753667 0.887086324176354
2.06759399270497 0.887086324176354
2.06925763491698 0.887086324176354
2.07467479776912 0.887086324176354
2.07669706189577 0.887086324176354
2.08318877700668 0.887086324176354
2.08946279995865 0.887086324176354
2.09529915812192 0.887086324176354
2.09818372955306 0.887086324176354
2.10406267301521 0.887086324176354
2.10849209826162 0.887086324176354
2.11287155121173 0.887086324176354
2.11965668409492 0.887086324176354
2.12811803984211 0.887086324176354
2.13151349537093 0.887086324176354
2.13335493395061 0.887086324176354
2.13866161439787 0.887086324176354
2.14651700041911 0.887086324176354
2.1523750572641 0.887086324176354
2.15371505036143 0.887086324176354
2.16421728424646 0.887086324176354
2.1663240982116 0.887086324176354
2.17184714281271 0.887086324176354
2.17345368276587 0.887086324176354
2.18169892518314 0.887086324176354
2.19051027261443 0.887086324176354
2.19188646308504 0.887086324176354
2.19567745042801 0.887086324176354
2.2004626981862 0.887086324176354
2.21082628714332 0.887086324176354
2.22250414228871 0.887086324176354
2.23077303993005 0.887086324176354
2.2323685451828 0.887086324176354
2.23749625419019 0.887086324176354
2.24309161430018 0.887086324176354
2.2440946060637 0.887086324176354
2.25315237937669 0.887086324176354
2.26545469813237 0.887086324176354
2.26733458055985 0.887086324176354
2.27965777393719 0.887086324176354
2.28143788661648 0.887086324176354
2.28739038745731 0.887086324176354
2.29191868065887 0.887086324176354
2.29890386428245 0.887086324176354
2.29987089115956 0.887086324176354
2.30516001918116 0.887086324176354
2.31402587753713 0.887086324176354
2.31542612549356 0.887086324176354
2.32318066994559 0.887086324176354
2.33113475758481 0.887086324176354
2.33722747595027 0.887086324176354
2.35077102490696 0.887086324176354
2.35679609643138 0.887086324176354
2.36135606440265 0.887086324176354
2.36788566981093 0.887086324176354
2.37132879914134 0.887086324176354
2.37973827778176 0.887086324176354
2.38479575751431 0.887086324176354
2.38859225636685 0.887086324176354
2.40038732044424 0.887086324176354
2.40640950249279 0.887086324176354
2.41037191876016 0.887086324176354
2.42175998499518 0.887086324176354
2.42921658216513 0.887086324176354
2.43483929957924 0.887086324176354
2.43634443442074 0.887086324176354
2.44695643658855 0.887086324176354
2.45654359141353 0.887086324176354
2.46292036321896 0.887086324176354
2.46658870798138 0.887086324176354
2.47830507449567 0.887086324176354
2.48073074863258 0.887086324176354
2.48606495877014 0.887086324176354
2.49975482456968 0.887086324176354
2.50986315182124 0.887086324176354
2.52417135263896 0.887086324176354
2.53266911202279 0.887086324176354
2.5410312582654 0.887086324176354
2.54394879149914 0.887086324176354
2.55551870370514 0.887086324176354
2.55970553464447 0.887086324176354
2.57434798040935 0.887086324176354
2.58199044604414 0.887086324176354
2.59008858742656 0.887086324176354
2.59592484965601 0.887086324176354
2.6096684081536 0.887086324176354
2.6180858518714 0.887086324176354
2.62163515590931 0.887086324176354
2.63610141255844 0.887086324176354
2.63994384841657 0.887086324176354
2.65930451982754 0.887086324176354
2.6629047240306 0.887086324176354
2.67811889172563 0.887086324176354
2.6894733985693 0.887086324176354
2.69638240567392 0.887086324176354
2.71234222357604 0.887086324176354
2.71315612791276 0.887086324176354
2.73660402249398 0.887086324176354
2.75502401344087 0.887086324176354
2.76497437597548 0.887086324176354
2.77672053946386 0.887086324176354
2.79159826377729 0.887086324176354
2.81800710301829 0.887086324176354
2.85623421931255 0.887086324176354
2.8740939273732 0.887086324176354
2.90336529939826 0.887086324176354
3.25130389102226 0.887086324176354
4.41529998165683 0.887086324176354
5.69482925258288 0.887086324176354
};
\addlegendentry{Eigenvalue}
\draw[draw=black,fill=steelblue31119180] (axis cs:0.0890613983470835,0) rectangle (axis cs:0.264241643791952,0.844844118263194);
\addlegendimage{ybar,ybar legend,draw=black,fill=steelblue31119180}
\addlegendentry{ESD}

\draw[draw=black,fill=steelblue31119180] (axis cs:0.264241643791952,0) rectangle (axis cs:0.439421889236821,0.736384400378054);
\draw[draw=black,fill=steelblue31119180] (axis cs:0.439421889236821,0) rectangle (axis cs:0.61460213468169,0.627924682492914);
\draw[draw=black,fill=steelblue31119180] (axis cs:0.61460213468169,0) rectangle (axis cs:0.789782380126558,0.530881777016737);
\draw[draw=black,fill=steelblue31119180] (axis cs:0.789782380126558,0) rectangle (axis cs:0.964962625571427,0.456672496358483);
\draw[draw=black,fill=steelblue31119180] (axis cs:0.964962625571427,0) rectangle (axis cs:1.1401428710163,0.405296840518154);
\draw[draw=black,fill=steelblue31119180] (axis cs:1.1401428710163,0) rectangle (axis cs:1.31532311646116,0.359629590882306);
\draw[draw=black,fill=steelblue31119180] (axis cs:1.31532311646116,0) rectangle (axis cs:1.49050336190603,0.325379153655419);
\draw[draw=black,fill=steelblue31119180] (axis cs:1.49050336190603,0) rectangle (axis cs:1.6656836073509,0.279711904019571);
\draw[draw=black,fill=steelblue31119180] (axis cs:1.6656836073509,0) rectangle (axis cs:1.84086385279577,0.251169872997166);
\draw[draw=black,fill=steelblue31119180] (axis cs:1.84086385279577,0) rectangle (axis cs:2.01604409824064,0.22262784197476);
\draw[draw=black,fill=steelblue31119180] (axis cs:2.01604409824064,0) rectangle (axis cs:2.19122434368551,0.194085810952356);
\draw[draw=black,fill=steelblue31119180] (axis cs:2.19122434368551,0) rectangle (axis cs:2.36640458913038,0.159835373725469);
\draw[draw=black,fill=steelblue31119180] (axis cs:2.36640458913038,0) rectangle (axis cs:2.54158483457524,0.137001748907545);
\draw[draw=black,fill=steelblue31119180] (axis cs:2.54158483457524,0) rectangle (axis cs:2.71676508002011,0.10845971788514);
\draw[draw=black,fill=steelblue31119180] (axis cs:2.71676508002011,0) rectangle (axis cs:2.89194532546498,0.0456672496358484);
\draw[draw=black,fill=steelblue31119180] (axis cs:2.89194532546498,0) rectangle (axis cs:3.06712557090985,0.00570840620448105);
\draw[draw=black,fill=steelblue31119180] (axis cs:3.06712557090985,0) rectangle (axis cs:3.24230581635472,0);
\draw[draw=black,fill=steelblue31119180] (axis cs:3.24230581635472,0) rectangle (axis cs:3.41748606179959,0.00570840620448105);
\draw[draw=black,fill=steelblue31119180] (axis cs:3.41748606179959,0) rectangle (axis cs:3.59266630724446,0);
\draw[draw=black,fill=steelblue31119180] (axis cs:3.59266630724446,0) rectangle (axis cs:3.76784655268933,0);
\draw[draw=black,fill=steelblue31119180] (axis cs:3.76784655268933,0) rectangle (axis cs:3.94302679813419,0);
\draw[draw=black,fill=steelblue31119180] (axis cs:3.94302679813419,0) rectangle (axis cs:4.11820704357906,0);
\draw[draw=black,fill=steelblue31119180] (axis cs:4.11820704357906,0) rectangle (axis cs:4.29338728902393,0);
\draw[draw=black,fill=steelblue31119180] (axis cs:4.29338728902393,0) rectangle (axis cs:4.4685675344688,0.00570840620448106);
\draw[draw=black,fill=steelblue31119180] (axis cs:4.4685675344688,0) rectangle (axis cs:4.64374777991367,0);
\draw[draw=black,fill=steelblue31119180] (axis cs:4.64374777991367,0) rectangle (axis cs:4.81892802535854,0);
\draw[draw=black,fill=steelblue31119180] (axis cs:4.81892802535854,0) rectangle (axis cs:4.99410827080341,0);
\draw[draw=black,fill=steelblue31119180] (axis cs:4.99410827080341,0) rectangle (axis cs:5.16928851624828,0);
\draw[draw=black,fill=steelblue31119180] (axis cs:5.16928851624828,0) rectangle (axis cs:5.34446876169314,0);
\draw[draw=black,fill=steelblue31119180] (axis cs:5.34446876169314,0) rectangle (axis cs:5.51964900713801,0);
\draw[draw=black,fill=steelblue31119180] (axis cs:5.51964900713801,0) rectangle (axis cs:5.69482925258288,0.00570840620448103);
\addplot [line width=1pt, darkorange25512714]
table {%
-0.191226994364706 -0
-0.160240337884006 -0
-0.129253681403306 -0
-0.0982670249226052 -0
-0.0672803684419048 -0
-0.0362937119612044 -0
-0.00530705548050403 -0
0.0256796010001964 0
0.0566662574808968 0
0.0876529139615971 0.263769251482256
0.118639570442298 0.813100919442043
0.149626226922998 0.893723715320903
0.180612883403698 0.897291749748691
0.211599539884399 0.87718330378631
0.242586196365099 0.849267840910418
0.2735728528458 0.819339459198725
0.3045595093265 0.789707686345867
0.3355461658072 0.761301556748638
0.366532822287901 0.734457795018756
0.397519478768601 0.709248368215063
0.428506135249301 0.685626540707922
0.459492791730002 0.663495091466083
0.490479448210702 0.642738999582436
0.521466104691403 0.623241150446999
0.552452761172103 0.604889665538387
0.583439417652803 0.587581022521393
0.614426074133504 0.57122105200228
0.645412730614204 0.555724881051006
0.676399387094905 0.541016380545262
0.707386043575605 0.527027407568998
0.738372700056305 0.513696993684572
0.769359356537006 0.500970554792928
0.800346013017706 0.488799157930137
0.831332669498407 0.477138858762128
0.862319325979107 0.465950112217526
0.893305982459807 0.45519725301596
0.924292638940508 0.444848040258327
0.955279295421208 0.434873259328727
0.986265951901908 0.425246374322514
1.01725260838261 0.415943224617636
1.04823926486331 0.40694175980129
1.07922592134401 0.398221807816816
1.11021257782471 0.389764871836354
1.14119923430541 0.38155395195879
1.17218589078611 0.37357338836567
1.20317254726681 0.365808723036795
1.23415920374751 0.358246577534643
1.26514586022821 0.350874544717644
1.29613251670891 0.343681092542966
1.32711917318961 0.33665547837605
1.35810582967031 0.329787672442684
1.38909248615101 0.32306828924522
1.42007914263171 0.316488525922479
1.45106579911241 0.310040106666891
1.48205245559311 0.303715232426132
1.51303911207382 0.297506535212804
1.54402576855452 0.291407036427128
1.57501242503522 0.285410108666169
1.60599908151592 0.27950944055058
1.63698573799662 0.273699004147558
1.66797239447732 0.267973024607834
1.69895905095802 0.262325951665796
1.72994570743872 0.256752432676075
1.76093236391942 0.25124728687734
1.79191902040012 0.245805480584919
1.82290567688082 0.240422103018094
1.85389233336152 0.235092342465271
1.88487898984222 0.22981146247996
1.91586564632292 0.224574777781839
1.94685230280362 0.219377629508557
1.97783895928432 0.214215359423546
2.00882561576502 0.209083282630142
2.03981227224572 0.203976658269152
2.07079892872642 0.198890657580592
2.10178558520712 0.193820328583758
2.13277224168782 0.188760556463589
2.16375889816852 0.183706018532261
2.19474555464922 0.178651132344605
2.22573221112992 0.173589995158051
2.25671886761062 0.168516312404706
2.28770552409132 0.163423312129991
2.31869218057203 0.158303641367891
2.34967883705273 0.15314923904523
2.38066549353343 0.147951178049634
2.41165215001413 0.142699466265621
2.44263880649483 0.137382792213402
2.47362546297553 0.131988194649659
2.50461211945623 0.126500625818243
2.53559877593693 0.12090236272817
2.56658543241763 0.115172195846654
2.59757208889833 0.109284282353297
2.62855874537903 0.10320647673675
2.65954540185973 0.0968978141979484
2.69053205834043 0.0903045538994154
2.72151871482113 0.0833536267275054
2.75250537130183 0.0759410475927157
2.78349202778253 0.0679095715177058
2.81447868426323 0.059000130538885
2.84546534074393 0.0487255598968897
2.87645199722463 0.0359229410957067
2.90743865370533 0.0151371122527123
2.93842531018603 0
2.96941196666673 0
3.00039862314743 0
3.03138527962813 0
3.06237193610883 0
3.09335859258954 0
3.12434524907024 0
3.15533190555094 0
3.18631856203164 0
3.21730521851234 0
3.24829187499304 0
3.27927853147374 0
3.31026518795444 0
3.34125184443514 0
3.37223850091584 0
3.40322515739654 0
3.43421181387724 0
3.46519847035794 0
3.49618512683864 0
3.52717178331934 0
3.55815843980004 0
3.58914509628074 0
3.62013175276144 0
3.65111840924214 0
3.68210506572284 0
3.71309172220354 0
3.74407837868424 0
3.77506503516494 0
3.80605169164564 0
3.83703834812634 0
3.86802500460704 0
3.89901166108775 0
3.92999831756845 0
3.96098497404915 0
3.99197163052985 0
4.02295828701055 0
4.05394494349125 0
4.08493159997195 0
4.11591825645265 0
4.14690491293335 0
4.17789156941405 0
4.20887822589475 0
4.23986488237545 0
4.27085153885615 0
4.30183819533685 0
4.33282485181755 0
4.36381150829825 0
4.39479816477895 0
4.42578482125965 0
4.45677147774035 0
4.48775813422105 0
4.51874479070175 0
4.54973144718245 0
4.58071810366315 0
4.61170476014385 0
4.64269141662455 0
4.67367807310525 0
4.70466472958596 0
4.73565138606666 0
4.76663804254736 0
4.79762469902806 0
4.82861135550876 0
4.85959801198946 0
4.89058466847016 0
4.92157132495086 0
4.95255798143156 0
4.98354463791226 0
5.01453129439296 0
5.04551795087366 0
5.07650460735436 0
5.10749126383506 0
5.13847792031576 0
5.16946457679646 0
5.20045123327716 0
5.23143788975786 0
5.26242454623856 0
5.29341120271926 0
5.32439785919996 0
5.35538451568066 0
5.38637117216136 0
5.41735782864206 0
5.44834448512276 0
5.47933114160347 0
5.51031779808417 0
5.54130445456487 0
5.57229111104557 0
5.60327776752627 0
5.63426442400697 0
5.66525108048767 0
5.69623773696837 0
5.72722439344907 0
5.75821104992977 0
5.78919770641047 0
5.82018436289117 0
5.85117101937187 0
5.88215767585257 0
5.91314433233327 0
5.94413098881397 0
5.97511764529467 0
};
\addlegendentry{MP}
\addplot [line width=1pt, forestgreen4416044, dashed]
table {%
3.3307389612582 0
3.3307389612582 0.942156337236126
};
\addlegendentry{Spike}
\addplot [line width=1pt, forestgreen4416044, dashed]
table {%
4.35643136148104 0
4.35643136148104 0.942156337236126
};
\addplot [line width=1pt, forestgreen4416044, dashed]
table {%
5.78721890149058 0
5.78721890149058 0.942156337236126
};
\end{axis}

\end{tikzpicture}

%% file: classification_accuracy.tex
% This file was created with tikzplotlib v0.10.1.
\begin{tikzpicture}

\definecolor{darkgray176}{RGB}{176,176,176}

\begin{axis}[
name=ax,
width=.8\linewidth,
height=.8\linewidth,
axis on top,
tick align=outside,
xmin=-0.5, xmax=9.5,
xtick pos=left,
xtick={0,1,2,3,4,5,6,7,8,9},
xticklabel style={rotate=45, anchor=east, font=\scriptsize},
xticklabels={T-shirt/top,Trouser,Pullover,Dress,Coat,Sandal,Shirt,Sneaker,Bag,Ankle boot},
y dir=reverse,
ymin=-0.5, ymax=9.5,
ytick pos=right,
ytick={0,1,2,3,4,5,6,7,8,9},
yticklabel style={anchor=west, font=\scriptsize},
yticklabels={T-shirt/top,Trouser,Pullover,Dress,Coat,Sandal,Shirt,Sneaker,Bag,Ankle boot}
]
\addplot graphics [includegraphics cmd=\pgfimage,xmin=-0.5, xmax=9.5, ymin=9.5, ymax=-0.5] {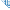};
\addplot graphics [includegraphics cmd=\pgfimage,xmin=-0.5, xmax=9.5, ymin=9.5, ymax=-0.5] {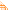};
\draw (axis cs:0,1) node[
  scale=0.8,
  text=black,
  rotate=0.0
]{0.61};
\draw (axis cs:1,0) node[
  scale=0.8,
  text=black,
  rotate=0.0
]{0.64};
\draw (axis cs:0,2) node[
  scale=0.8,
  text=black,
  rotate=0.0
]{0.55};
\draw (axis cs:2,0) node[
  scale=0.8,
  text=black,
  rotate=0.0
]{0.54};
\draw (axis cs:0,3) node[
  scale=0.8,
  text=black,
  rotate=0.0
]{0.56};
\draw (axis cs:3,0) node[
  scale=0.8,
  text=black,
  rotate=0.0
]{0.57};
\draw (axis cs:0,4) node[
  scale=0.8,
  text=black,
  rotate=0.0
]{0.61};
\draw (axis cs:4,0) node[
  scale=0.8,
  text=black,
  rotate=0.0
]{0.61};
\draw (axis cs:0,5) node[
  scale=0.8,
  text=black,
  rotate=0.0
]{0.9};
\draw (axis cs:5,0) node[
  scale=0.8,
  text=black,
  rotate=0.0
]{0.89};
\draw (axis cs:0,6) node[
  scale=0.8,
  text=black,
  rotate=0.0
]{0.52};
\draw (axis cs:6,0) node[
  scale=0.8,
  text=black,
  rotate=0.0
]{0.53};
\draw (axis cs:0,7) node[
  scale=0.8,
  text=black,
  rotate=0.0
]{0.83};
\draw (axis cs:7,0) node[
  scale=0.8,
  text=black,
  rotate=0.0
]{0.85};
\draw (axis cs:0,8) node[
  scale=0.8,
  text=black,
  rotate=0.0
]{0.5};
\draw (axis cs:8,0) node[
  scale=0.8,
  text=black,
  rotate=0.0
]{0.48};
\draw (axis cs:0,9) node[
  scale=0.8,
  text=black,
  rotate=0.0
]{0.52};
\draw (axis cs:9,0) node[
  scale=0.8,
  text=black,
  rotate=0.0
]{0.54};
\draw (axis cs:1,2) node[
  scale=0.8,
  text=black,
  rotate=0.0
]{0.69};
\draw (axis cs:2,1) node[
  scale=0.8,
  text=black,
  rotate=0.0
]{0.72};
\draw (axis cs:1,3) node[
  scale=0.8,
  text=black,
  rotate=0.0
]{0.54};
\draw (axis cs:3,1) node[
  scale=0.8,
  text=black,
  rotate=0.0
]{0.54};
\draw (axis cs:1,4) node[
  scale=0.8,
  text=black,
  rotate=0.0
]{0.76};
\draw (axis cs:4,1) node[
  scale=0.8,
  text=black,
  rotate=0.0
]{0.81};
\draw (axis cs:1,5) node[
  scale=0.8,
  text=black,
  rotate=0.0
]{0.97};
\draw (axis cs:5,1) node[
  scale=0.8,
  text=black,
  rotate=0.0
]{0.95};
\draw (axis cs:1,6) node[
  scale=0.8,
  text=black,
  rotate=0.0
]{0.58};
\draw (axis cs:6,1) node[
  scale=0.8,
  text=black,
  rotate=0.0
]{0.58};
\draw (axis cs:1,7) node[
  scale=0.8,
  text=black,
  rotate=0.0
]{0.86};
\draw (axis cs:7,1) node[
  scale=0.8,
  text=black,
  rotate=0.0
]{0.86};
\draw (axis cs:1,8) node[
  scale=0.8,
  text=black,
  rotate=0.0
]{0.67};
\draw (axis cs:8,1) node[
  scale=0.8,
  text=black,
  rotate=0.0
]{0.69};
\draw (axis cs:1,9) node[
  scale=0.8,
  text=black,
  rotate=0.0
]{0.71};
\draw (axis cs:9,1) node[
  scale=0.8,
  text=black,
  rotate=0.0
]{0.71};
\draw (axis cs:2,3) node[
  scale=0.8,
  text=black,
  rotate=0.0
]{0.63};
\draw (axis cs:3,2) node[
  scale=0.8,
  text=black,
  rotate=0.0
]{0.63};
\draw (axis cs:2,4) node[
  scale=0.8,
  text=black,
  rotate=0.0
]{0.54};
\draw (axis cs:4,2) node[
  scale=0.8,
  text=black,
  rotate=0.0
]{0.55};
\draw (axis cs:2,5) node[
  scale=0.8,
  text=black,
  rotate=0.0
]{0.9};
\draw (axis cs:5,2) node[
  scale=0.8,
  text=black,
  rotate=0.0
]{0.88};
\draw (axis cs:2,6) node[
  scale=0.8,
  text=black,
  rotate=0.0
]{0.57};
\draw (axis cs:6,2) node[
  scale=0.8,
  text=black,
  rotate=0.0
]{0.57};
\draw (axis cs:2,7) node[
  scale=0.8,
  text=black,
  rotate=0.0
]{0.84};
\draw (axis cs:7,2) node[
  scale=0.8,
  text=black,
  rotate=0.0
]{0.85};
\draw (axis cs:2,8) node[
  scale=0.8,
  text=black,
  rotate=0.0
]{0.56};
\draw (axis cs:8,2) node[
  scale=0.8,
  text=black,
  rotate=0.0
]{0.56};
\draw (axis cs:2,9) node[
  scale=0.8,
  text=black,
  rotate=0.0
]{0.59};
\draw (axis cs:9,2) node[
  scale=0.8,
  text=black,
  rotate=0.0
]{0.61};
\draw (axis cs:3,4) node[
  scale=0.8,
  text=black,
  rotate=0.0
]{0.69};
\draw (axis cs:4,3) node[
  scale=0.8,
  text=black,
  rotate=0.0
]{0.7};
\draw (axis cs:3,5) node[
  scale=0.8,
  text=black,
  rotate=0.0
]{0.92};
\draw (axis cs:5,3) node[
  scale=0.8,
  text=black,
  rotate=0.0
]{0.9};
\draw (axis cs:3,6) node[
  scale=0.8,
  text=black,
  rotate=0.0
]{0.54};
\draw (axis cs:6,3) node[
  scale=0.8,
  text=black,
  rotate=0.0
]{0.53};
\draw (axis cs:3,7) node[
  scale=0.8,
  text=black,
  rotate=0.0
]{0.81};
\draw (axis cs:7,3) node[
  scale=0.8,
  text=black,
  rotate=0.0
]{0.82};
\draw (axis cs:3,8) node[
  scale=0.8,
  text=black,
  rotate=0.0
]{0.59};
\draw (axis cs:8,3) node[
  scale=0.8,
  text=black,
  rotate=0.0
]{0.59};
\draw (axis cs:3,9) node[
  scale=0.8,
  text=black,
  rotate=0.0
]{0.59};
\draw (axis cs:9,3) node[
  scale=0.8,
  text=black,
  rotate=0.0
]{0.58};
\draw (axis cs:4,5) node[
  scale=0.8,
  text=black,
  rotate=0.0
]{0.95};
\draw (axis cs:5,4) node[
  scale=0.8,
  text=black,
  rotate=0.0
]{0.93};
\draw (axis cs:4,6) node[
  scale=0.8,
  text=black,
  rotate=0.0
]{0.62};
\draw (axis cs:6,4) node[
  scale=0.8,
  text=black,
  rotate=0.0
]{0.63};
\draw (axis cs:4,7) node[
  scale=0.8,
  text=black,
  rotate=0.0
]{0.91};
\draw (axis cs:7,4) node[
  scale=0.8,
  text=black,
  rotate=0.0
]{0.91};
\draw (axis cs:4,8) node[
  scale=0.8,
  text=black,
  rotate=0.0
]{0.61};
\draw (axis cs:8,4) node[
  scale=0.8,
  text=black,
  rotate=0.0
]{0.64};
\draw (axis cs:4,9) node[
  scale=0.8,
  text=black,
  rotate=0.0
]{0.66};
\draw (axis cs:9,4) node[
  scale=0.8,
  text=black,
  rotate=0.0
]{0.7};
\draw (axis cs:5,6) node[
  scale=0.8,
  text=black,
  rotate=0.0
]{0.85};
\draw (axis cs:6,5) node[
  scale=0.8,
  text=black,
  rotate=0.0
]{0.85};
\draw (axis cs:5,7) node[
  scale=0.8,
  text=black,
  rotate=0.0
]{0.76};
\draw (axis cs:7,5) node[
  scale=0.8,
  text=black,
  rotate=0.0
]{0.76};
\draw (axis cs:5,8) node[
  scale=0.8,
  text=black,
  rotate=0.0
]{0.89};
\draw (axis cs:8,5) node[
  scale=0.8,
  text=black,
  rotate=0.0
]{0.88};
\draw (axis cs:5,9) node[
  scale=0.8,
  text=black,
  rotate=0.0
]{0.93};
\draw (axis cs:9,5) node[
  scale=0.8,
  text=black,
  rotate=0.0
]{0.92};
\draw (axis cs:6,7) node[
  scale=0.8,
  text=black,
  rotate=0.0
]{0.76};
\draw (axis cs:7,6) node[
  scale=0.8,
  text=black,
  rotate=0.0
]{0.79};
\draw (axis cs:6,8) node[
  scale=0.8,
  text=black,
  rotate=0.0
]{0.53};
\draw (axis cs:8,6) node[
  scale=0.8,
  text=black,
  rotate=0.0
]{0.53};
\draw (axis cs:6,9) node[
  scale=0.8,
  text=black,
  rotate=0.0
]{0.52};
\draw (axis cs:9,6) node[
  scale=0.8,
  text=black,
  rotate=0.0
]{0.52};
\draw (axis cs:7,8) node[
  scale=0.8,
  text=black,
  rotate=0.0
]{0.79};
\draw (axis cs:8,7) node[
  scale=0.8,
  text=black,
  rotate=0.0
]{0.8};
\draw (axis cs:7,9) node[
  scale=0.8,
  text=black,
  rotate=0.0
]{0.85};
\draw (axis cs:9,7) node[
  scale=0.8,
  text=black,
  rotate=0.0
]{0.85};
\draw (axis cs:8,9) node[
  scale=0.8,
  text=black,
  rotate=0.0
]{0.51};
\draw (axis cs:9,8) node[
  scale=0.8,
  text=black,
  rotate=0.0
]{0.51};
\end{axis}

\node [anchor=south] at (ax.north) {Observed};
\node [rotate=90, anchor=south] at (ax.west) {Predicted};
\end{tikzpicture}